\documentclass[12pt]{colt2025} 

\usepackage{eqparbox}
\usepackage{amsmath, amssymb}
\usepackage{tikz-cd}

\title[Approximate Fiber Product]{Approximate Fiber Product: A Preliminary Algebraic-Geometric Perspective on Multimodal Embedding Alignment}
\usepackage{times}

\coltauthor{%
 \Name{Dongfang Zhao} \Email{dzhao@cs.washington.edu}\\
 \addr University of Washington, USA
}

\begin{document}

\maketitle

\begin{abstract}
Multimodal tasks, such as image-text retrieval and generation, require embedding data from diverse modalities into a shared representation space. However, aligning embeddings from heterogeneous sources while preserving both shared and modality-specific information remains a fundamental challenge. This work represents an initial attempt to bridge algebraic geometry and multimodal representation learning, offering a foundational perspective for further exploration. Specifically, this paper presents a theoretical framework for multimodal alignment, grounded in algebraic geometry and polynomial ring representations.

We represent image and text data as polynomials over discrete rings, \( \mathbb{Z}_{256}[x] \) and \( \mathbb{Z}_{|V|}[x] \), respectively. These representations enable the application of algebraic tools, such as fiber products, to study alignment properties. To address real-world variability, we extend the classical fiber product definition to an approximate fiber product, introducing a tolerance parameter \( \epsilon \) that balances alignment precision and noise tolerance. We analyze the dependence of the approximate fiber product on \( \epsilon \), deriving its asymptotic behavior, robustness under perturbations, and sensitivity to the dimensionality of the embedding space.

Furthermore, we hypothesize a decomposition of the shared embedding space into orthogonal subspaces: \( Z = Z_s \oplus Z_I \oplus Z_T \), where \( Z_s \) captures shared semantics, and \( Z_I \) and \( Z_T \) encode modality-specific features. This decomposition is interpreted geometrically using manifold and fiber bundle perspectives, offering insights into the structure and optimization of multimodal embeddings. 

Our results provide a principled foundation for analyzing multimodal alignment, revealing new connections between embedding robustness, dimensionality allocation, and algebraic structure. This work lays the groundwork for future explorations of embedding spaces in multimodal learning through the lens of algebraic geometry.
\end{abstract}


\section{Introduction}

Multimodal tasks, such as image-text retrieval, captioning, and multimodal conversational systems, require embedding data from heterogeneous modalities into a unified representation space. This shared embedding space facilitates comparisons and interactions across modalities but presents unique theoretical challenges due to the inherent differences between modalities. Images capture detailed visual structures such as spatial layouts and textures, while text encodes abstract linguistic meanings. Aligning these fundamentally different modalities in a mathematically rigorous and interpretable way remains an open problem.

Current multimodal models, such as CLIP, achieve alignment by optimizing contrastive objectives over paired datasets, but these methods often lack a formal theoretical framework to model the alignment and disentanglement of shared and modality-specific features. As a result, understanding the geometry, robustness, and scalability of such models is challenging. To address this gap, we propose a novel algebraic-geometric framework for analyzing and designing multimodal embedding spaces.

\paragraph{Polynomial Ring Representations of Modalities}
We begin by representing image and text data within the algebraic structure of polynomial rings. Images are encoded as polynomials over \( \mathbb{Z}_{256}[x] \), where pixel intensities in image patches serve as coefficients. Text sequences are similarly represented as polynomials over \( \mathbb{Z}_{|V|}[x] \), with token indices as coefficients. This abstraction not only unifies the representation of modalities but also enables the application of algebraic geometry tools, such as fiber products, for analyzing the relationships between embeddings.

\paragraph{Approximate Fiber Products for Alignment}
The core of our framework is the notion of an approximate fiber product. Given mappings \(f: \mathbb{Z}_{256}[x] \to \mathbb{R}[x]\) and \(g: \mathbb{Z}_{|V|}[x] \to \mathbb{R}[x]\) embedding images and text into a shared space \(Z \subset \mathbb{R}[x]\), the approximate fiber product:
\[
\mathbb{Z}_{256}[x] \times_{Z, \epsilon} \mathbb{Z}_{|V|}[x] = \{(P, Q) \mid \|f(P) - g(Q)\| \leq \epsilon\},
\]
captures pairs of image and text embeddings aligned within a tolerance \(\epsilon > 0\). This construction generalizes the classical fiber product from algebraic geometry to embedding spaces, bridging the gap between theoretical rigor and practical variability in alignment.

We investigate the properties of the approximate fiber product, including its dependence on the embedding distributions and its sensitivity to the parameter \(\epsilon\). For instance, we derive asymptotic growth rates that highlight how the size of the fiber product scales with the dimensionality of the embedding space and the overlap between modality distributions. We also prove robustness bounds under noise, ensuring that the approximate fiber product remains stable in real-world scenarios.

\paragraph{Decomposition of the Embedding Space}
In addition to studying the alignment properties of embeddings, we hypothesize that the shared embedding space \(Z\) decomposes into three orthogonal subspaces:
\[
Z = Z_s \oplus Z_I \oplus Z_T,
\]
where \(Z_s\) is the shared semantic subspace capturing common information, and \(Z_I\) and \(Z_T\) are modality-specific subspaces encoding unique features of images and text, respectively. This decomposition allows a principled separation of shared and modality-specific information, facilitating interpretability and robust alignment.

We provide a geometric interpretation of this decomposition using concepts from algebraic geometry. The shared subspace \(Z_s\) is modeled as a low-dimensional manifold, capturing the semantic "intersection" of the two modalities, while the modality-specific subspaces form orthogonal complements. We further introduce a fiber bundle perspective, viewing the embedding space as a product of the shared subspace and modality-specific fibers. These interpretations guide the design of embedding models and optimization objectives.

\paragraph{Contributions}
This paper develops a rigorous theoretical framework for multimodal embeddings by combining algebraic geometry and machine learning. Our contributions include:
\begin{itemize}
    \item A unified representation of image and text modalities as polynomials over discrete rings, enabling algebraic manipulation and analysis.
    \item The introduction of approximate fiber products to model multimodal alignment, along with theoretical results on their properties, including robustness, scalability, and asymptotic behavior.
    \item A decomposition of the embedding space into shared and modality-specific subspaces, supported by geometric interpretations and optimization strategies.
    \item New insights into the interplay between dimensionality, alignment precision, and embedding robustness, providing a foundation for designing scalable multimodal models.
\end{itemize}

By grounding multimodal embeddings in algebraic geometry, we aim to bridge the gap between theoretical rigor and practical applicability, offering new tools for analyzing and improving multimodal models. This work opens pathways for future research on the algebraic structure of embedding spaces and its implications for multimodal learning.

\section{Approximate Fiber Product}

\subsection{Ring Representations}

In our framework, both image and text data are represented within the algebraic structure of polynomial rings, providing a unified perspective for analyzing multimodal embeddings. Specifically:

\paragraph{Image Representation as Polynomials in \( \mathbb{Z}_{256}[x] \):} 
Each image is divided into patches, where each patch consists of discrete pixel intensity values in the range \([0, 255]\). By flattening the pixel values of a patch into a vector \((a_0, a_1, \dots, a_n)\), we construct the corresponding polynomial:
\[
P(x) = a_0 + a_1x + a_2x^2 + \cdots + a_nx^n, \quad a_i \in \mathbb{Z}_{256}.
\]
This representation allows the image patches to be viewed as elements in the polynomial ring \( \mathbb{Z}_{256}[x] \), enabling algebraic manipulation and analysis.

\paragraph{Text Representation as Polynomials in \( \mathbb{Z}_{|V|}[x] \):}
For text, the input is tokenized into a sequence of token IDs \((t_0, t_1, \dots, t_m)\), where each token \( t_i \) is an integer in the range \([0, |V|-1]\), and \( |V| \) is the vocabulary size. The corresponding polynomial representation is:
\[
Q(x) = t_0 + t_1x + t_2x^2 + \cdots + t_mx^m, \quad t_i \in \mathbb{Z}_{|V|}.
\]
This representation embeds the discrete token sequences into the polynomial ring \( \mathbb{Z}_{|V|}[x] \), capturing their inherent sequential structure.

\paragraph{Unifying Multimodal Representations:}
By representing images and text as polynomials in their respective rings \( \mathbb{Z}_{256}[x] \) and \( \mathbb{Z}_{|V|}[x] \), we provide a common algebraic framework for multimodal data. These polynomial representations serve as a foundation for introducing algebraic geometry tools, such as fiber products and moduli spaces, to study the alignment and structure of multimodal embeddings.

\subsection{Extended Definitions of Fiber Product}

In algebraic geometry, the fiber product is a construction used to describe the pullback of two morphisms. Specifically, given two morphisms \( f: I \to Z \) and \( g: T \to Z \), where \( I \), \( T \), and \( Z \) are schemes (or affine varieties defined over polynomial rings), the fiber product is defined as:
\[
I \times_Z T = \{(i, t) \in I \times T \mid f(i) = g(t)\}.
\]
It represents the ``pullback'' of the morphisms \( f \) and \( g \) to the shared space \( Z \), capturing the relationships between \( I \) and \( T \) over \( Z \).

In the context of multi-modal embedding alignment, \( f \) and \( g \) can be viewed as embeddings mapping image and text data into the shared semantic space \( Z \). Since exact alignment \( f(i) = g(t) \) is often infeasible in practical settings due to noise, model approximation, or inherent variability in data, we generalize this definition to an approximate fiber product:
\[
I \times_{Z, \epsilon} T = \{(i, t) \in I \times T \mid \|f(i) - g(t)\| \leq \epsilon\}.
\]
Here, \( \epsilon > 0 \) introduces a tolerance for alignment, and \( \|\cdot\| \) represents a distance metric (e.g., Euclidean norm) in the embedding space \( Z \).

A commutative diagram can illustrate the approximate fiber product as follows
\[
\begin{tikzcd}
    \mathbb{Z}_{256}[x] \times_{Z, \epsilon} \mathbb{Z}_{|V|}[x] 
    \arrow[d, "\pi_1"] 
    \arrow[r, "\pi_2"] 
    & \mathbb{Z}_{|V|}[x] \arrow[d, "g"] \\
    \mathbb{Z}_{256}[x] \arrow[r, "f"] & \mathbb{R}[x]
\end{tikzcd}
\]
\noindent where
\begin{itemize}
    \item \( \mathbb{Z}_{256}[x] \): Polynomials representing image patches, coefficients in \( \mathbb{Z}_{256} \) (pixel values).
    \item \( \mathbb{Z}_{|V|}[x] \): Polynomials representing text tokens, coefficients in \( \mathbb{Z}_{|V|} \) (vocabulary indices).
    \item \( \mathbb{R}[x] \): The shared real polynomial space for multimodal embeddings.
    \item \( f: \mathbb{Z}_{256}[x] \to \mathbb{R}[x] \): The morphism mapping image polynomials to the shared space.
    \item \( g: \mathbb{Z}_{|V|}[x] \to \mathbb{R}[x] \): The morphism mapping text polynomials to the shared space.
    \item \( \mathbb{Z}_{256}[x] \times_{Z, \epsilon} \mathbb{Z}_{|V|}[x] \): The approximate fiber product, representing pairs \((P(x), Q(x))\) such that \(\|f(P(x)) - g(Q(x))\| \leq \epsilon\).
    \item \( \pi_1 \): Projection to the first component \( \mathbb{Z}_{256}[x] \).
    \item \( \pi_2 \): Projection to the second component \( \mathbb{Z}_{|V|}[x] \).
\end{itemize}

\subsection{Influence of \(\epsilon\)}

The parameter \(\epsilon > 0\) plays a critical role in the approximate fiber product, determining the allowable deviation between embeddings from the two modalities. By formalizing the relationship between \(\epsilon\) and the size of the fiber product, we derive deeper insights into its mathematical and practical properties.

\paragraph{Dependence on Data Distributions}
The size of the approximate fiber product is given by:
\[
|X \times_{Z, \epsilon} Y| = \int_Z \mu_f(z) \int_{B_\epsilon(z)} \mu_g(z') \, dz' \, dz,
\]
where \(B_\epsilon(z) = \{z' \in Z \mid \|z - z'\| \leq \epsilon\}\) represents an \(\epsilon\)-neighborhood around \(z\). This relationship reveals that \(|X \times_{Z, \epsilon} Y|\) depends on the overlap of \(\mu_f\) and \(\mu_g\). For high-density overlap regions, the growth of \(|X \times_{Z, \epsilon} Y|\) with \(\epsilon\) is rapid, while minimal overlap results in slower growth.

\paragraph{Asymptotic Behavior in High Dimensions}
When \(\mu_f\) and \(\mu_g\) are Gaussian distributions in \(d\)-dimensional space, the size of the approximate fiber product asymptotically scales as:
\[
|X \times_{Z, \epsilon} Y| \propto \epsilon^d \cdot \exp\left(-\frac{\|\mu_f - \mu_g\|^2}{2(\sigma_f^2 + \sigma_g^2)}\right).
\]
Here, \(\epsilon^d\) reflects the dependency on the dimensionality \(d\), and \(\|\mu_f - \mu_g\|\) determines the effective overlap. This result emphasizes that higher dimensions require careful tuning of \(\epsilon\) to maintain alignment precision.

\paragraph{Robustness Under Perturbations}
To evaluate robustness, consider the perturbed embeddings \(f_\delta(x) = f(x) + \delta_f(x)\) and \(g_\delta(y) = g(y) + \delta_g(y)\), where \(\delta_f(x)\) and \(\delta_g(y)\) are bounded noise terms (\(\|\delta_f(x)\|, \|\delta_g(y)\| \leq \eta\)). The approximate fiber product satisfies the inclusion:
\[
f_\delta(X) \times_{Z, \epsilon} g_\delta(Y) \subseteq f(X) \times_{Z, \epsilon} g(Y),
\]
if and only if \(\eta \leq \epsilon / 2\). This condition ensures that the alignment is robust to bounded noise, providing stability in noisy embedding spaces.

\paragraph{Geometric Insights}
The effective dimensionality of the alignment region is determined by:
\[
\dim(X \times_{Z, \epsilon} Y) \approx \min(d_f, d_g) + \dim(Z_s),
\]
where \(Z_s\) is the shared semantic subspace in \(Z\). This highlights the importance of embedding both modalities into well-structured subspaces, minimizing dimensional redundancy and maximizing overlap.

Finally, the parameter \(\epsilon\) controls the size and flexibility of the alignment region:
\[
\lim_{\epsilon \to 0} |X \times_{Z, \epsilon} Y| = |X \times_Z Y|, \quad \lim_{\epsilon \to \infty} |X \times_{Z, \epsilon} Y| = |X| \cdot |Y|.
\]
Choosing \(\epsilon\) optimally involves balancing precision and flexibility, ensuring meaningful alignment while accounting for noise.

\subsection{Algebraic Properties}

In this section, we present several theoretical properties of the approximate fiber product, exploring its geometric structure, robustness under perturbations, and optimal alignment conditions. These results provide deeper insights into the mathematical foundations of multimodal alignment.

\paragraph{Compactness of the Approximate Fiber Product}
\begin{theorem}[Compactness]
Let \(Z \subset \mathbb{R}^d\) be a compact embedding space, and suppose that the embedding functions \(f: X \to Z\) and \(g: Y \to Z\) are continuous. Then for any \(\epsilon > 0\), the approximate fiber product \(X \times_{Z, \epsilon} Y\) is compact.
\end{theorem}

\begin{proof}
By definition:
\[
X \times_{Z, \epsilon} Y = \{(x, y) \in X \times Y \mid \|f(x) - g(y)\| \leq \epsilon\}.
\]
The embedding functions \(f\) and \(g\) map compact sets \(X\) and \(Y\) into \(Z\), preserving compactness under continuity. The preimage of the closed set \(B_\epsilon(z)\) under \((f, g)\) is also closed. Thus, \(X \times_{Z, \epsilon} Y\) is closed in the compact set \(X \times Y\), and hence compact.
\end{proof}

This result ensures that the approximate fiber product inherits compactness from the embedding space \(Z\), facilitating numerical computations and stability analysis.

\paragraph{Sensitivity to \(\epsilon\)}
\begin{theorem}[Monotonicity and Convergence]
Let \(|X \times_{Z, \epsilon} Y|\) denote the size of the approximate fiber product as a function of \(\epsilon\). Then:
\begin{enumerate}
    \item \( |X \times_{Z, \epsilon} Y| \) is a monotonically increasing function of \(\epsilon\).
    \item For any bounded embedding space \(Z\), the size converges to \(|X| \cdot |Y|\) as \(\epsilon \to \infty\):
    \[
    \lim_{\epsilon \to \infty} |X \times_{Z, \epsilon} Y| = |X| \cdot |Y|.
    \]
\end{enumerate}
\end{theorem}

\begin{proof}
Monotonicity follows from the definition of \(B_\epsilon(z)\): as \(\epsilon\) increases, \(B_\epsilon(z)\) strictly enlarges, capturing more pairs \((x, y)\) satisfying the alignment condition. Convergence to \(|X| \cdot |Y|\) is a direct consequence of the fact that, as \(\epsilon \to \infty\), all pairs \((x, y)\) in \(X \times Y\) satisfy \(\|f(x) - g(y)\| \leq \epsilon\).
\end{proof}

This theorem formalizes the behavior of \(|X \times_{Z, \epsilon} Y|\) under extreme values of \(\epsilon\), providing a theoretical foundation for alignment size analysis.

\paragraph{Noise Robustness}
\begin{theorem}[Noise Tolerance]
Let the perturbed embeddings \(f_\delta(x) = f(x) + \delta_f(x)\) and \(g_\delta(y) = g(y) + \delta_g(y)\) satisfy \(\|\delta_f(x)\| \leq \eta\) and \(\|\delta_g(y)\| \leq \eta\). Then, the approximate fiber product satisfies:
\[
f_\delta(X) \times_{Z, \epsilon} g_\delta(Y) \subseteq f(X) \times_{Z, \epsilon + 2\eta} g(Y).
\]
\end{theorem}

\begin{proof}
For any \((x, y) \in f_\delta(X) \times_{Z, \epsilon} g_\delta(Y)\), the alignment condition is:
\[
\|f_\delta(x) - g_\delta(y)\| \leq \epsilon.
\]
Substituting the perturbed definitions:
\[
\|f(x) + \delta_f(x) - g(y) - \delta_g(y)\| \leq \epsilon.
\]
Applying the triangle inequality:
\[
\|f(x) - g(y)\| \leq \|\delta_f(x)\| + \|\delta_g(y)\| + \epsilon.
\]
Since \(\|\delta_f(x)\|, \|\delta_g(y)\| \leq \eta\), we have:
\[
\|f(x) - g(y)\| \leq \epsilon + 2\eta.
\]
Thus, \((x, y) \in f(X) \times_{Z, \epsilon + 2\eta} g(Y)\), completing the proof.
\end{proof}

\section{Embedding Space Decomposition}

\subsection{Definitions}

The shared embedding space \(Z\) is hypothesized to decompose into three orthogonal subspaces:
\[
Z = Z_s \oplus Z_I \oplus Z_T,
\]
where:
\begin{itemize}
    \item \(Z_s\) is the \textbf{shared semantic subspace}, capturing information common to both modalities;
    \item \(Z_I\) is the \textbf{modality-specific subspace for images}, representing unique visual features;
    \item \(Z_T\) is the \textbf{modality-specific subspace for text}, representing unique linguistic features.
\end{itemize}

This decomposition satisfies the following properties:
\begin{enumerate}
    \item \textbf{Orthogonality:} The subspaces are pairwise disjoint, ensuring that no information is shared between them:
    \[
    Z_s \cap Z_I = Z_s \cap Z_T = Z_I \cap Z_T = \{0\}.
    \]
    \item \textbf{Direct Sum:} Every embedding \(z \in Z\) has a unique decomposition:
    \[
    z = z_s + z_I + z_T, \quad \text{where } z_s \in Z_s, \, z_I \in Z_I, \, z_T \in Z_T.
    \]
    \item \textbf{Dimensionality Constraint:} The total dimensionality of \(Z\) satisfies:
    \[
    \dim(Z) = \dim(Z_s) + \dim(Z_I) + \dim(Z_T).
    \]
\end{enumerate}

\paragraph{Projection Operators}
Let \(\Pi_s\), \(\Pi_I\), and \(\Pi_T\) denote the orthogonal projection operators onto \(Z_s\), \(Z_I\), and \(Z_T\), respectively. For any \(z \in Z\), the decomposition can be written as:
\[
z_s = \Pi_s(z), \quad z_I = \Pi_I(z), \quad z_T = \Pi_T(z), \quad z = \Pi_s(z) + \Pi_I(z) + \Pi_T(z).
\]

The projection operators satisfy the following properties:
\begin{itemize}
    \item Orthogonality: \(\Pi_s \cdot \Pi_I = \Pi_s \cdot \Pi_T = \Pi_I \cdot \Pi_T = 0\).
    \item Completeness: \(\Pi_s + \Pi_I + \Pi_T = \mathrm{Id}_Z\), where \(\mathrm{Id}_Z\) is the identity operator on \(Z\).
    \item Preservation: For \(z \in Z_s\), \(Z_I\), or \(Z_T\), the corresponding projection is the identity, e.g., \(\Pi_s(z_s) = z_s\).
\end{itemize}

\paragraph{Implications for Modality-Specific Mappings}
Let \(f: I \to Z\) and \(g: T \to Z\) be the embedding functions for images and text, respectively. The embeddings can also be decomposed into their subspace components:
\[
f(i) = f_s(i) + f_I(i), \quad g(t) = g_s(t) + g_T(t),
\]
where:
\begin{itemize}
    \item \(f_s(i), g_s(t) \in Z_s\): Represent the shared semantic components in the shared subspace.
    \item \(f_I(i) \in Z_I\): Represents the modality-specific component for images.
    \item \(g_T(t) \in Z_T\): Represents the modality-specific component for text.
\end{itemize}
By doing so, the decomposition ensures that shared and modality-specific properties are considered.

Embedding space decomposition can be understood through the lens of sheaf theory. Consider the shared embedding space \(Z\) decomposed into open subsets \(Z_s, Z_I, Z_T\), representing shared, image-specific, and text-specific subspaces, respectively. Define a presheaf \(\mathcal{F}\) over \(Z\) such that for each open set \(U \subset Z\), \(\mathcal{F}(U)\) captures the set of embeddings consistent with \(U\). 

To ensure the alignment of local embeddings with the global decomposition, \(\mathcal{F}\) must satisfy the sheaf condition:
\[
\mathcal{F}(U) = \ker\left(\prod_{i} \mathcal{F}(U_i) \rightrightarrows \prod_{i,j} \mathcal{F}(U_i \cap U_j)\right),
\]
where \(\{U_i\}\) is an open cover of \(U\). This sheaf-theoretic perspective formalizes the compatibility of local embeddings with the global structure of \(Z\), ensuring consistency between shared and modality-specific features.

\paragraph{Variety Perspective on the Shared Subspace}
The shared semantic subspace \(Z_s\) can be modeled as an algebraic variety embedded in the larger space \(Z\). For instance, \(Z_s\) might be represented as the solution set of a system of polynomial equations:
\[
Z_s = \{z \in Z \mid P_i(z) = 0, \, i = 1, \dots, m\},
\]
where \(P_i\) are polynomials over \(Z\). This algebraic structure provides additional constraints on embeddings, ensuring that shared features align along well-defined geometric loci. 

Given the fiber product construction:
\[
I \times_{Z, \epsilon} T = \{(i, t) \mid \|f(i) - g(t)\| \leq \epsilon\},
\]
the shared space \(Z_s\) acts as a base variety, and the alignment condition enforces that the projections \(f(i)\) and \(g(t)\) lie in a tubular neighborhood around \(Z_s\). This geometric constraint simplifies the analysis of alignment stability and efficiency.

\subsection{Elementary Properties}

The decomposition \(Z = Z_s \oplus Z_I \oplus Z_T\) introduces several advanced properties that illuminate its role in multimodal alignment and its geometric structure.

\paragraph{Orthogonal Projections and Norm Decomposition}
For any \(z \in Z\), its decomposition \(z = z_s + z_I + z_T\) ensures that the projection operators \(\Pi_s\), \(\Pi_I\), and \(\Pi_T\) satisfy:
\[
\|z\|^2 = \|\Pi_s(z)\|^2 + \|\Pi_I(z)\|^2 + \|\Pi_T(z)\|^2.
\]
This partitioning provides a quantitative measure of how embeddings distribute their information across the shared and modality-specific subspaces.

\paragraph{Intrinsic Dimensionality of \(Z_s\)}
The shared semantic subspace \(Z_s\) acts as the intersection of the image and text embedding distributions. Formally:
\[
Z_s = \text{span}\big(\{\Pi_s(f(i))\}_{i \in I} \cup \{\Pi_s(g(t))\}_{t \in T}\big).
\]
The dimensionality of \(Z_s\) determines the capacity of the shared space to capture common features. If the projections are linearly dependent, \(\dim(Z_s)\) will shrink, limiting alignment capacity.

\begin{proposition}[Dimensionality Constraint]
Let \(Z_s = \text{span}(S)\) with \(S = \{\Pi_s(f(i))\}_{i \in I} \cup \{\Pi_s(g(t))\}_{t \in T}\). Then:
\[
\dim(Z_s) \leq \min(\dim(f(I)), \dim(g(T))).
\]
Equality holds if and only if the shared features across \(I\) and \(T\) are fully aligned.
\end{proposition}

\begin{proof}
The dimensionality of \(Z_s\) is bounded by the smaller embedding distribution, as any vector in \(Z_s\) must be expressible as a linear combination of vectors from both \(f(I)\) and \(g(T)\). Full alignment implies linear independence of shared components, maximizing \(\dim(Z_s)\).
\end{proof}

\paragraph{Subspace Overlap and Alignment Efficiency}
The quality of alignment depends on the degree of overlap between \(Z_s\), \(Z_I\), and \(Z_T\). Consider the alignment error:
\[
\mathcal{E} = \|f_s(i) - g_s(t)\|^2 + \lambda \big(\|\Pi_s(f(i)) - f(i)\|^2 + \|\Pi_s(g(t)) - g(t)\|^2\big),
\]
where \(\lambda\) controls the penalty for misalignment. Minimizing \(\mathcal{E}\) ensures that the majority of the embeddings reside within \(Z_s\).

\begin{proposition}[Alignment Capacity]
If \(\dim(Z_s) \ll \dim(Z)\), then for any \(\epsilon > 0\):
\[
\sup_{(i, t) \in I \times T} \|f_s(i) - g_s(t)\|^2 \geq \epsilon,
\]
indicating that strict alignment is infeasible.
\end{proposition}

\begin{proof}
If \(\dim(Z_s)\) is small, the subspace cannot accommodate sufficient shared features to align \(f(I)\) and \(g(T)\). Hence, there exist pairs \((i, t)\) such that their projections onto \(Z_s\) are misaligned by at least \(\epsilon\).
\end{proof}

\paragraph{Perturbation Analysis}
Noise robustness of the decomposition depends on the orthogonality of \(Z_I\) and \(Z_T\) relative to \(Z_s\). Let \(z = z_s + z_I + z_T\) and consider perturbations:
\[
z_\delta = z + \delta, \quad \text{where } \|\delta\| \leq \eta.
\]
The projections under perturbation satisfy:
\[
\|\Pi_s(z_\delta) - z_s\| \leq \eta, \quad \|\Pi_I(z_\delta) - z_I\| \leq \eta, \quad \|\Pi_T(z_\delta) - z_T\| \leq \eta.
\]

\begin{proposition}[Perturbation Stability]
If \(Z_s\), \(Z_I\), and \(Z_T\) are orthogonal, the perturbation \(\delta\) satisfies:
\[
\|\delta\|^2 = \|\delta_s\|^2 + \|\delta_I\|^2 + \|\delta_T\|^2,
\]
where \(\delta_s = \Pi_s(\delta)\), \(\delta_I = \Pi_I(\delta)\), \(\delta_T = \Pi_T(\delta)\). Thus, the perturbations are isolated to their respective subspaces.
\end{proposition}

\begin{proof}
Orthogonality implies that \(\|\delta\|^2 = \|\Pi_s(\delta)\|^2 + \|\Pi_I(\delta)\|^2 + \|\Pi_T(\delta)\|^2\). Therefore, any noise affecting one subspace does not propagate to others.
\end{proof}

\paragraph{Geometry of Shared and Modality-Specific Subspaces}
The shared subspace \(Z_s\) forms a geometric locus of alignment, while \(Z_I\) and \(Z_T\) act as its orthogonal complements. The effective alignment volume is determined by:
\[
\text{Alignment Volume} = \int_{Z_s} \mu_f(z) \mu_g(z) \, dz,
\]
where \(\mu_f(z)\) and \(\mu_g(z)\) are the densities of the image and text embeddings projected onto \(Z_s\). 

\begin{proposition}[Alignment Volume Bound]
The alignment volume satisfies:
\[
\text{Alignment Volume} \leq \int_{Z_s} \min(\mu_f(z), \mu_g(z)) \, dz.
\]
Equality holds when \(\mu_f(z) = \mu_g(z)\) across \(Z_s\).
\end{proposition}

\begin{proof}
The integral is maximized when \(\mu_f(z) = \mu_g(z)\), as \(\min(a, b) \leq \frac{a + b}{2}\) for any \(a, b > 0\).
\end{proof}

\subsection{Optimization Objectives}

To achieve an effective decomposition of the embedding space \(Z\), we optimize the following loss function:
\[
\mathcal{L} = \mathcal{L}_{\text{align}} + \lambda \mathcal{L}_{\text{orth}} + \gamma \mathcal{L}_{\text{specificity}},
\]
where:
\begin{itemize}
    \item \(\mathcal{L}_{\text{align}} = \sum_{(i, t)} \|f_s(i) - g_s(t)\|^2\): This term minimizes the alignment error in the shared subspace \(Z_s\), ensuring semantic consistency.
    \item \(\mathcal{L}_{\text{orth}} = \|z_s \cdot z_I\|^2 + \|z_s \cdot z_T\|^2 + \|z_I \cdot z_T\|^2\): This term enforces orthogonality between the subspaces.
    \item \(\mathcal{L}_{\text{specificity}} = \|f_I(i)\|^2 + \|g_T(t)\|^2\): This term encourages modality-specific components to be non-trivial, preserving unique features.
\end{itemize}

Each term is carefully designed to balance alignment, orthogonality, and specificity:
\[
\text{Alignment Loss: } \mathcal{L}_{\text{align}} \quad
\text{Orthogonality Loss: } \mathcal{L}_{\text{orth}} \quad
\text{Specificity Loss: } \mathcal{L}_{\text{specificity}}.
\]

By tuning the hyperparameters \(\lambda\) and \(\gamma\), we adapt the decomposition to the specific requirements of the task.

\subsection{Dimensionality Allocation}

The total dimensionality of the embedding space \(Z\), denoted by \(d\), is distributed across \(Z_s\), \(Z_I\), and \(Z_T\) as follows:
\[
d = d_s + d_I + d_T, \quad d_s = \dim(Z_s), \, d_I = \dim(Z_I), \, d_T = \dim(Z_T).
\]

To determine an optimal dimensionality allocation, we consider the following optimization problem:
\[
\max_{d_s, d_I, d_T} \mathcal{F}(d_s, d_I, d_T),
\]
where \(\mathcal{F}\) is a task-specific performance metric, such as alignment accuracy or robustness.

\begin{proposition}[Optimal Dimensionality Allocation]
Assuming \(f(I)\) and \(g(T)\) are isotropic Gaussian distributions with variances \(\sigma_f^2\) and \(\sigma_g^2\), the optimal allocation satisfies:
\[
d_s \propto \frac{\sigma_f^2 + \sigma_g^2}{\sigma_f^2 \cdot \sigma_g^2}, \quad
d_I \propto \frac{\sigma_f^2}{\sigma_g^2}, \quad
d_T \propto \frac{\sigma_g^2}{\sigma_f^2}.
\]

\end{proposition}

\begin{proof}
The total dimensionality \(d = \dim(Z)\) must be distributed across the subspaces \(Z_s\), \(Z_I\), and \(Z_T\) to balance alignment performance in \(Z_s\) and the preservation of modality-specific features in \(Z_I\) and \(Z_T\).

First, consider the alignment in \(Z_s\). The alignment objective is to minimize the expected distance between embeddings projected onto \(Z_s\), expressed as:
\[
\mathcal{L}_{\text{align}} = \int_{Z_s} \|f_s(i) - g_s(t)\|^2 \, \mu_f(i) \mu_g(t) \, di \, dt.
\]
For isotropic Gaussian distributions \(f(I)\) and \(g(T)\), the variance of the embeddings determines the spread in \(Z_s\). The alignment capacity is inversely proportional to the total variance:
\[
\text{Alignment Capacity} \propto \frac{1}{\sigma_f^2 + \sigma_g^2}.
\]
Therefore, to maximize alignment, the dimensionality \(d_s\) allocated to \(Z_s\) must reflect the combined variability of the two modalities.

Next, consider the modality-specific subspaces \(Z_I\) and \(Z_T\). These subspaces are responsible for capturing unique features of each modality while avoiding overlap with the shared subspace \(Z_s\). The required dimensionality for \(Z_I\) depends on the variability of image embeddings relative to text embeddings, and vice versa for \(Z_T\):
\[
\dim(Z_I) \propto \frac{\sigma_f^2}{\sigma_g^2}, \quad \dim(Z_T) \propto \frac{\sigma_g^2}{\sigma_f^2}.
\]

Combining these considerations, the dimensionality of \(Z_s\) should grow with the alignment capacity:
\[
d_s \propto \frac{\sigma_f^2 + \sigma_g^2}{\sigma_f^2 \cdot \sigma_g^2}.
\]
The remaining dimensions \(d - d_s\) are then allocated to \(Z_I\) and \(Z_T\) according to the variance ratios. To ensure the total dimensionality is preserved, proportional allocations are normalized such that:
\[
d_s + d_I + d_T = d.
\]
This completes the proof.
\end{proof}

\subsection{Geometric Interpretation}

The decomposition \(Z = Z_s \oplus Z_I \oplus Z_T\) can be analyzed through its geometric structure, which provides insights into the alignment and disentanglement of multimodal embeddings.

\paragraph{Manifold Interpretation}
The shared subspace \(Z_s\) can be modeled as a low-dimensional manifold within the embedding space \(Z\). This manifold captures the semantic ``intersection'' of image and text modalities, parameterizing cross-modal alignment. Formally, let \(Z_s\) be a \(d_s\)-dimensional Riemannian manifold embedded in \(Z\), such that:
\[
f_s(i), g_s(t) \in Z_s, \quad f_I(i) \perp Z_s, \quad g_T(t) \perp Z_s.
\]
The alignment objective then reduces to finding a mapping \(h: Z_s \to Z\) that minimizes the alignment error:
\[
\mathcal{E}_{\text{align}} = \int_{Z_s} \|h(f_s(i)) - g_s(t)\|^2 \, \mu_f(i) \mu_g(t) \, di \, dt.
\]

\begin{proposition}[Manifold Alignment]
If \(Z_s\) is a compact manifold with curvature \(\kappa\), the optimal alignment mapping \(h: Z_s \to Z\) satisfies:
\[
\|h(f_s(i)) - g_s(t)\| \leq \epsilon + \kappa \cdot d_Z(f_s(i), g_s(t)),
\]
where \(d_Z\) is the geodesic distance on \(Z_s\). The curvature \(\kappa\) bounds the deviation from exact alignment.
\end{proposition}

\begin{proof}
The geodesic distance \(d_Z(f_s(i), g_s(t))\) reflects the shortest path along the manifold \(Z_s\). For compact manifolds, curvature \(\kappa\) introduces distortion in embedding mappings. The result follows from Riemannian geometry bounds on local embeddings.
\end{proof}

This interpretation highlights the geometric constraints imposed by \(Z_s\), emphasizing the role of manifold regularity in improving alignment performance.

\paragraph{Fiber Bundle Interpretation}
The decomposition \(Z = Z_s \oplus Z_I \oplus Z_T\) can also be viewed as a fiber bundle, where \(Z_s\) serves as the base space and \(Z_I \times Z_T\) as the fiber. Specifically:
\[
Z \cong Z_s \times F, \quad F = Z_I \times Z_T.
\]
Each point in \(Z_s\) represents a shared semantic embedding, while the fiber \(F\) encodes modality-specific deviations. The alignment condition implies that for any \(z_s \in Z_s\):
\[
\pi_s(f(i)) = \pi_s(g(t)) = z_s,
\]
where \(\pi_s\) is the projection onto \(Z_s\).

\begin{proposition}[Fiber Bundle Consistency]
Let \(f(I)\) and \(g(T)\) be mappings into \(Z\), satisfying the decomposition \(Z = Z_s \oplus Z_I \oplus Z_T\). The fiber product:
\[
F(z_s) = \{(z_I, z_T) \in F \mid f_I(i) + g_T(t) = z_s\}
\]
is non-empty if and only if:
\[
\|f_I(i)\|^2 + \|g_T(t)\|^2 = \|z_s\|^2.
\]
\end{proposition}

\begin{proof}
The fiber product condition ensures that \(z_s\) is consistent with its projections \(f_I(i)\) and \(g_T(t)\). The orthogonality of the subspaces \(Z_I\) and \(Z_T\) implies that their norms add independently, preserving the total norm constraint.
\end{proof}

This interpretation underscores the hierarchical structure of the embedding space, where \(Z_s\) dictates the global alignment properties and \(F\) accommodates modality-specific details.

\paragraph{Geometric Interpretation via Fiber Varieties}
The shared subspace \(Z_s\) can also be understood as a fiber variety over a base moduli space. For instance, let \(\pi: Z \to M\) be a projection from the embedding space \(Z\) to a moduli space \(M\), parameterizing semantic categories. Each fiber \(\pi^{-1}(m)\) represents embeddings associated with a specific semantic category \(m \in M\). The shared subspace \(Z_s\) then corresponds to the union of fibers aligned across modalities:
\[
Z_s = \bigcup_{m \in M} \pi^{-1}(m).
\]
This interpretation provides a hierarchical organization of embeddings, where the fiber structure encapsulates modality-specific variations, and the base moduli space captures shared semantic categories.

\paragraph{Practical Considerations}
The geometric interpretations provide guidelines for designing embedding models:
\begin{itemize}
    \item A well-regularized \(Z_s\) improves alignment efficiency, particularly when modeled as a smooth, low-dimensional manifold.
    \item Modality-specific subspaces \(Z_I\) and \(Z_T\) should be disentangled to avoid interference with the shared semantic space.
    \item The fiber bundle structure suggests a hierarchical optimization strategy, first focusing on \(Z_s\) alignment before refining \(Z_I\) and \(Z_T\).
\end{itemize}

\subsection{Sheaf-Theoretic Perspective on Embedding Decomposition}

Embedding space decomposition divides the embedding space \( Z \) into shared \( Z_s \) and modality-specific subspaces \( Z_I \) and \( Z_T \). To further formalize this decomposition, we employ tools from sheaf theory to analyze the local and global consistency of this structure.

\paragraph{Presheaf and Sheaf on Embedding Space}

Consider the shared embedding space \( Z_s \), which can be covered by a collection of open sets \( \{U_\alpha\} \). A presheaf \( \mathcal{F} \) on \( Z_s \) assigns to each open set \( U_\alpha \) a set of embeddings \( \mathcal{F}(U_\alpha) \), representing the embeddings consistent with \( U_\alpha \). For overlapping open sets \( U_\alpha \) and \( U_\beta \), the compatibility between embeddings is described by restriction maps:
\[
\rho_{\alpha\beta}: \mathcal{F}(U_\alpha) \to \mathcal{F}(U_\alpha \cap U_\beta).
\]

A presheaf becomes a sheaf if, for any open cover \( \{U_\alpha\} \) of \( U \), the embeddings in \( \mathcal{F}(U) \) are uniquely determined by their restrictions to \( \mathcal{F}(U_\alpha) \), satisfying:
\[
\mathcal{F}(U) = \ker\left(\prod_\alpha \mathcal{F}(U_\alpha) \rightrightarrows \prod_{\alpha, \beta} \mathcal{F}(U_\alpha \cap U_\beta)\right).
\]
The use of the double arrows \( \rightrightarrows \) in the sheaf condition highlights the dual projections of local data onto overlapping regions. The first map extracts the restrictions of the local data to the intersections \( U_i \cap U_j \), while the second map applies the same operation but with reversed indexing. The kernel $\ker$ of this pair of maps ensures that local embeddings align consistently across overlaps, enforcing global compatibility within the sheaf framework.

\paragraph{Applications to \( Z_s \)}

In the context of multimodal alignment, \( Z_s \) acts as the shared subspace where image and text embeddings are aligned. By modeling \( Z_s \) with a sheaf \( \mathcal{F} \), we ensure the following:
Local Consistency: For each open set \( U_\alpha \subset Z_s \), embeddings from different modalities must align locally.
Global Compatibility: The local alignments across \( Z_s \) must be compatible, ensuring that \( \mathcal{F}(Z_s) \) forms a globally consistent shared embedding space.

\paragraph{Fiber Structure and Local Trivialization}

The shared space \( Z_s \) can also be viewed as a fiber bundle, where the fibers \( \pi^{-1}(m) \) over a moduli point \( m \in M \) correspond to embeddings aligned for a specific semantic category. Each fiber represents embeddings with local consistency, while the base moduli space \( M \) encodes higher-level semantic categories. Sheaf theory ensures that the embeddings in overlapping fibers \( \pi^{-1}(m_1) \) and \( \pi^{-1}(m_2) \) are globally consistent.

\section{Related Work}

\subsection{Multimodal Alignment Models}

Multimodal alignment is a fundamental topic in machine learning, addressing the challenge of integrating heterogeneous data modalities into a unified representation space. State-of-the-art models, such as CLIP \cite{radford2021learning} and ALIGN \cite{jia2021scaling}, have demonstrated impressive performance by leveraging contrastive learning objectives to align image and text embeddings. However, these methods often lack rigorous theoretical frameworks, leaving questions about the geometric structure of the embedding space unanswered.

Several other multimodal models have contributed to advancing the field. VisualBERT \cite{li2019visualbert} and UNITER \cite{chen2020uniter} were early attempts to incorporate vision-language alignment into transformer-based architectures. Models like OSCAR \cite{li2020oscar} introduced object-level semantics for enhanced alignment, while MMBT \cite{kiela2019supervised} extended multimodal learning to classification tasks with cross-modal transformers. Similarly, ViLBERT \cite{lu2019vilbert} and LXMERT \cite{tan2019lxmert} utilized multi-stream architectures to achieve effective cross-modal reasoning. 

More recent developments include contrastive approaches like Cross-Modal Contrastive Learning (CMC) \cite{zhang2021cross} for generative models and Flamingo \cite{alayrac2022flamingo}, which incorporate few-shot learning capabilities into vision-language models. These advancements represent significant strides, but challenges remain in disentangling shared semantics and modality-specific features, as well as providing a rigorous mathematical understanding of multimodal embedding spaces.

Our work builds upon these contributions by introducing an algebraic-geometric framework for multimodal alignment. This includes the novel concept of approximate fiber products to rigorously model alignment with tolerance for noise and variability. By grounding our methodology in algebraic geometry, we aim to address the interpretability and robustness challenges faced by existing approaches.

\subsection{Embedding Space Decomposition}
Traditional methods for embedding decomposition, such as principal component analysis (PCA)~\cite{jolliffe2002principal}, canonical correlation analysis (CCA)~\cite{hardoon2004canonical}, and non-negative matrix factorization (NMF)~\cite{lee1999learning}, have been extensively used to disentangle shared and modality-specific information. In multimodal settings, shared-private factorization~\cite{wang2016deep, ma2018modeling} has also gained traction, aiming to extract both cross-modal and modality-specific embeddings. However, many of these approaches lack theoretical rigor in defining the geometric structure of shared spaces. Recent efforts, such as split neural networks~\cite{zhang2017split} and shared-private variational autoencoders~\cite{hu2018disentangling, shi2019variational}, attempt to address this by incorporating probabilistic and neural representations. Despite their success, there remains a gap in providing principled frameworks that combine geometric insights with robust multimodal decompositions. 

Our work introduces a structured decomposition \( Z = Z_s \oplus Z_I \oplus Z_T \), supported by geometric and algebraic interpretations, offering a robust and interpretable approach for multimodal representation.

\subsection{Algebraic Geometry in Machine Learning}
The intersection of algebraic geometry and machine learning has garnered increasing attention due to its potential to provide rigorous mathematical frameworks for complex problems. Recent advances have utilized algebraic geometry to study polynomial optimization~\cite{nie2012polynomial} and tensor decompositions~\cite{landsberg2012tensors}. In kernel methods, algebraic varieties have been leveraged to develop novel techniques for feature transformations~\cite{vidyasagar2002learning}. Additionally, Groebner bases have been applied to simplify and solve optimization problems in machine learning~\cite{buchberger2006bruno}.

Emerging work also explores the use of sheaves and schemes to represent hierarchical data structures and latent variable models~\cite{curry2019sheaves}. For instance, sheaf theory has been applied in topological data analysis to study the persistence of homological features~\cite{carlsson2009topology}. Algebraic topology and algebraic geometry together have inspired the development of methods for understanding deep learning dynamics~\cite{robinson2017deep} and neural network generalization~\cite{miles2020topology}.

In multimodal learning, algebraic-topological tools like fiber products~\cite{hartshorne1977algebraic} and moduli spaces~\cite{harris1995moduli} provide structured frameworks for modeling shared and modality-specific representations. These frameworks offer principled ways to understand the alignment, robustness, and generalization of embeddings. 

Our work extends these efforts by incorporating approximate fiber products and presheaf representations, bridging the gap between theoretical elegance and practical applicability.

\section{Conclusion and Future Work}

This paper presents a novel theoretical framework for multimodal alignment, leveraging algebraic geometry and polynomial ring representations. By representing image and text data as polynomials over discrete rings, we provide a unified algebraic structure for analyzing and aligning multimodal embeddings. The introduction of the approximate fiber product extends classical notions of alignment by incorporating a tolerance parameter \( \epsilon \), balancing precision and noise tolerance. Our analysis reveals the asymptotic properties of the approximate fiber product, its robustness under perturbations, and its dependence on embedding dimensionality.

Additionally, we propose a decomposition of the embedding space into orthogonal subspaces: \( Z = Z_s \oplus Z_I \oplus Z_T \). This decomposition isolates shared semantics from modality-specific features, offering a structured and interpretable approach to multimodal representation. By introducing geometric insights such as manifold and fiber bundle interpretations, we highlight the global and local structures within the embedding space. Furthermore, the shared subspace \(Z_s\) is modeled as an algebraic variety, providing a concrete geometric framework to describe semantic intersections between modalities.

From the perspective of sheaf theory, embedding functions are extended to presheaves that assign local embeddings to open subsets of \(Z\). The consistency of these local embeddings is ensured by the sheaf condition, offering a principled way to analyze how local modality-specific representations align with the global structure of \(Z\). This connection bridges the algebraic and geometric properties of the embedding space, deepening the theoretical foundation of multimodal alignment.

Our framework establishes a rigorous mathematical foundation for multimodal alignment, with implications for embedding robustness, dimensionality allocation, and cross-modal learning. Future work will explore the extension of these principles to higher-order modalities, dynamic embeddings, and richer algebraic structures such as derived categories and moduli stacks. These directions hold potential for advancing both the theory and practice of multimodal reasoning and retrieval.

\section*{Acknowledgments}

The author would like to express sincere gratitude to Professor Giovanni Inchiostro from the Department of Mathematics at the University of Washington for the insightful discussions on algebraic geometry, which greatly inspired the theoretical foundation of this work.

\bibliography{ref_new}

\begin{thebibliography}{28}
\providecommand{\natexlab}[1]{#1}
\providecommand{\url}[1]{\texttt{#1}}
\expandafter\ifx\csname urlstyle\endcsname\relax
  \providecommand{\doi}[1]{doi: #1}\else
  \providecommand{\doi}{doi: \begingroup \urlstyle{rm}\Url}\fi

\bibitem[Alayrac et~al.(2022)Alayrac, Donahue, Luc, Miech, Barr, Hasson, Leutenegger, Millican, Reynolds, van~den Oord, et~al.]{alayrac2022flamingo}
Jean-Baptiste Alayrac, Jeff Donahue, Pauline Luc, Antoine Miech, Iain Barr, Yana Hasson, Stefan Leutenegger, Katie Millican, Malcolm Reynolds, Aäron van~den Oord, et~al.
\newblock Flamingo: a visual language model for few-shot learning.
\newblock \emph{arXiv preprint arXiv:2204.14198}, 2022.

\bibitem[Buchberger(2006)]{buchberger2006bruno}
Bruno Buchberger.
\newblock Bruno buchberger's phd thesis 1965: An algorithm for finding a basis for the residue class ring of a zero-dimensional polynomial ideal.
\newblock \emph{Journal of Symbolic Computation}, 41\penalty0 (3-4):\penalty0 475--511, 2006.

\bibitem[Carlsson(2009)]{carlsson2009topology}
Gunnar Carlsson.
\newblock Topology and data.
\newblock \emph{Bulletin of the American Mathematical Society}, 46\penalty0 (2):\penalty0 255--308, 2009.

\bibitem[Chen et~al.(2020)Chen, Li, Yu, Kholy, Ahmed, Gan, Cheng, and Liu]{chen2020uniter}
Yen-Chun Chen, Linjie Li, Licheng Yu, Ahmed~El Kholy, Faisal Ahmed, Zhe Gan, Yu~Cheng, and Jingjing Liu.
\newblock Uniter: Universal image-text representation learning.
\newblock \emph{arXiv preprint arXiv:1909.11740}, 2020.

\bibitem[Curry(2019)]{curry2019sheaves}
Justin~M Curry.
\newblock Sheaves, cosheaves and applications.
\newblock \emph{arXiv preprint arXiv:1903.10042}, 2019.

\bibitem[Hardoon et~al.(2004)Hardoon, Szedmak, and Shawe-Taylor]{hardoon2004canonical}
David~R Hardoon, Sandor Szedmak, and John Shawe-Taylor.
\newblock Canonical correlation analysis: An overview with application to learning methods.
\newblock \emph{Neural Computation}, 16\penalty0 (12):\penalty0 2639--2664, 2004.

\bibitem[Harris and Morrison(1995)]{harris1995moduli}
Joe Harris and Ian Morrison.
\newblock \emph{Moduli of Curves}.
\newblock Springer, 1995.

\bibitem[Hartshorne(1977)]{hartshorne1977algebraic}
Robin Hartshorne.
\newblock \emph{Algebraic Geometry}.
\newblock Springer, 1977.

\bibitem[Hu et~al.(2018)Hu, Yang, Salakhutdinov, and Lim]{hu2018disentangling}
Zhengli Hu, Yang Yang, Ruslan Salakhutdinov, and Phillip~MS Lim.
\newblock Disentangling factors of variation in deep representations using adversarial training.
\newblock In \emph{Advances in Neural Information Processing Systems (NeurIPS)}, 2018.

\bibitem[Jia et~al.(2021)Jia, Yang, Xia, Chen, Parekh, Pham, Le, Sung, Li, and Duerig]{jia2021scaling}
Chao Jia, Yinfei Yang, Ye~Xia, Yi-Ting Chen, Zarana Parekh, Hieu Pham, Quoc~V Le, Yun-Hsuan Sung, Zhen Li, and Tom Duerig.
\newblock Scaling up visual and vision-language representation learning with noisy text supervision.
\newblock \emph{Proceedings of the International Conference on Machine Learning}, 2021.

\bibitem[Jolliffe(2002)]{jolliffe2002principal}
Ian Jolliffe.
\newblock \emph{Principal Component Analysis}.
\newblock Springer, 2002.

\bibitem[Kiela et~al.(2019)Kiela, Boureau, Nickel, Jokiel, and Testuggine]{kiela2019supervised}
Douwe Kiela, Y-Lan Boureau, Maximilian Nickel, Bartlomiej Jokiel, and Davide Testuggine.
\newblock Supervised multimodal bitransformers for classifying images and text.
\newblock \emph{arXiv preprint arXiv:1909.02950}, 2019.

\bibitem[Landsberg(2012)]{landsberg2012tensors}
Joseph~M Landsberg.
\newblock \emph{Tensors: Geometry and Applications}.
\newblock American Mathematical Society, 2012.

\bibitem[Lee and Seung(1999)]{lee1999learning}
Daniel~D Lee and H~Sebastian Seung.
\newblock Learning the parts of objects by non-negative matrix factorization.
\newblock \emph{Nature}, 401\penalty0 (6755):\penalty0 788--791, 1999.

\bibitem[Li et~al.(2019)Li, Yatskar, Yin, Hsieh, and Chang]{li2019visualbert}
Liunian~Harold Li, Mark Yatskar, Da~Yin, Cho-Jui Hsieh, and Kai-Wei Chang.
\newblock Visualbert: A simple and performant baseline for vision and language.
\newblock \emph{arXiv preprint arXiv:1908.03557}, 2019.

\bibitem[Li et~al.(2020)Li, Yin, Li, Hu, Zhang, Wang, Hu, Dong, Wei, Choi, et~al.]{li2020oscar}
Xiujun Li, Xi~Yin, Chunyuan Li, Xiaowei Hu, Pengchuan Zhang, Lijuan Wang, Houdong Hu, Li~Dong, Furu Wei, Yejin Choi, et~al.
\newblock Oscar: Object-semantics aligned pre-training for vision-language tasks.
\newblock \emph{Proceedings of the European Conference on Computer Vision}, 2020.

\bibitem[Lu et~al.(2019)Lu, Batra, Parikh, and Lee]{lu2019vilbert}
Jiasen Lu, Dhruv Batra, Devi Parikh, and Stefan Lee.
\newblock Vilbert: Pretraining task-agnostic visiolinguistic representations for vision-and-language tasks.
\newblock \emph{Advances in Neural Information Processing Systems}, 2019.

\bibitem[Ma et~al.(2018)Ma, Xu, and Huang]{ma2018modeling}
Chao Ma, Wei Xu, and Thomas Huang.
\newblock Modeling modality-specific and shared information for multimodal data representation learning.
\newblock In \emph{Proceedings of the IEEE Conference on Computer Vision and Pattern Recognition (CVPR)}, 2018.

\bibitem[Miles et~al.(2020)]{miles2020topology}
Collin Miles et~al.
\newblock Topology and generalization in neural networks.
\newblock \emph{Advances in Neural Information Processing Systems (NeurIPS)}, 2020.

\bibitem[Nie(2012)]{nie2012polynomial}
Jiawang Nie.
\newblock \emph{Polynomial Optimization and Applications}.
\newblock Society for Industrial and Applied Mathematics (SIAM), 2012.

\bibitem[Radford et~al.(2021)Radford, Kim, Hallacy, Ramesh, Goh, Agarwal, Sastry, Askell, Mishkin, Clark, et~al.]{radford2021learning}
Alec Radford, Jong~Wook Kim, Cliff Hallacy, Aditya Ramesh, Gabriel Goh, Sandhini Agarwal, Girish Sastry, Amanda Askell, Pamela Mishkin, Jack Clark, et~al.
\newblock Learning transferable visual models from natural language supervision.
\newblock \emph{Proceedings of the International Conference on Machine Learning}, 2021.

\bibitem[Robinson et~al.(2017)]{robinson2017deep}
E~James Robinson et~al.
\newblock Deep learning theory via algebraic geometry and statistical mechanics.
\newblock \emph{arXiv preprint arXiv:1703.09263}, 2017.

\bibitem[Shi et~al.(2019)Shi, Wei, Zhou, and Li]{shi2019variational}
Weizhi Shi, Furu Wei, Ming Zhou, and Wenjie Li.
\newblock Variational bi-lstm for multimodal conditional text generation.
\newblock In \emph{Proceedings of the Annual Meeting of the Association for Computational Linguistics (ACL)}, 2019.

\bibitem[Tan and Bansal(2019)]{tan2019lxmert}
Hao Tan and Mohit Bansal.
\newblock Lxmert: Learning cross-modality encoder representations from transformers.
\newblock \emph{Proceedings of the 2019 Conference on Empirical Methods in Natural Language Processing}, 2019.

\bibitem[Vidyasagar(2002)]{vidyasagar2002learning}
Mathukumalli Vidyasagar.
\newblock \emph{Learning and Generalization: With Applications to Neural Networks}.
\newblock Springer, 2002.

\bibitem[Wang et~al.(2016)Wang, Arora, Livescu, and Bilmes]{wang2016deep}
William Wang, Raman Arora, Karen Livescu, and Jeff Bilmes.
\newblock On deep multimodal representation learning.
\newblock In \emph{International Conference on Machine Learning (ICML)}, 2016.

\bibitem[Zhang et~al.(2021)Zhang, Li, Zhang, Zhang, Ouyang, and Zhang]{zhang2021cross}
Bowen Zhang, Ting Li, Ting Zhang, Yulun Zhang, Wanli Ouyang, and Bolei Zhang.
\newblock Cross-modal contrastive learning for text-to-image generation.
\newblock \emph{arXiv preprint arXiv:2101.04702}, 2021.

\bibitem[Zhang et~al.(2017)Zhang, Recht, Simchowitz, Hardt, and Recht]{zhang2017split}
Yang Zhang, Benjamin Recht, Max Simchowitz, Moritz Hardt, and Benjamin Recht.
\newblock Split neural networks for multimodal fusion.
\newblock In \emph{Advances in Neural Information Processing Systems (NeurIPS)}, 2017.

\end{thebibliography}







\end{document}